\documentclass[sigconf]{acmart}

\usepackage{amsmath,amsfonts}

\def\gA{{\mathcal{A}}}

\def\gO{{\mathcal{O}}}

\def\gS{{\mathcal{S}}}
\def\gT{{\mathcal{T}}}

\newcommand{\ie}{\textit{i.e.}}
\newcommand{\eg}{\textit{e.g.}}

\newcommand{\ttt}{\texttt}

\usepackage{multirow}
\usepackage{subcaption}
\usepackage{float}
\usepackage{graphicx}
\usepackage{caption}
\usepackage{amsmath}
\usepackage{booktabs}
\usepackage{algorithm}
\usepackage[noend]{algpseudocode}
\usepackage{bbm}
\usepackage{makecell}
\definecolor{lightgrey}{rgb}{0.95, 0.95, 0.95}
\definecolor{grey}{rgb}{0.4, 0.4, 0.4}
\newcommand{\ours}{{StepAgent}}
\newtheorem{assumption}{Assumption}
\usepackage[skins, breakable]{tcolorbox}
\newtcolorbox{cvbox}[1][]{
    enhanced,
    after skip=8mm,%   enlarge distance to the next box
    title=#1,
    breakable = true,
    fonttitle=\sffamily\bfseries\color{white},
    coltitle=white,
    colbacktitle=black!75!white, %gray!10,   % <- defines background color in title
    titlerule= 0pt,         % <- sets rule underneath title 
    overlay={%
        \ifcase\tcbsegmentstate
        \or%
        \else%
        \fi%
    }
    colback = gray!5!white,         % <- defines background color in box
    colframe = black!75     % <- defines color of frame
    }
\AtBeginDocument{%
  }

\setcopyright{acmlicensed}
\copyrightyear{2024}
\acmYear{2024}
\acmDOI{XXXXXXX.XXXXXXX}

\acmConference[WWW '25]{The Web Conference}{April 28-May 02, 2025}{Sydney, Australia}
\acmISBN{978-1-4503-XXXX-X/18/06}

\begin{document}

%%
%% The "title" command has an optional parameter,
%% allowing the author to define a "short title" to be used in page headers.
\title{From Novice to Expert: LLM Agent Policy Optimization via Step-wise Reinforcement Learning}

%%
%% The "author" command and its associated commands are used to define
%% the authors and their affiliations.
%% Of note is the shared affiliation of the first two authors, and the
%% "authornote" and "authornotemark" commands
%% used to denote shared contribution to the research.

\author{Zhirui Deng}
%\affiliation{
%  Gaoling School of Artificial Intelligence\\ Renmin University of China
%  \country{China}
%}
\orcid{0000-0001-8952-7666}
\email{zrdeng@ruc.edu.cn}

\author{Zhicheng Dou}
\authornote{Zhicheng Dou and Ruibin Xiong are the corresponding authors. }
\orcid{0000-0002-9781-948X}
\email{dou@ruc.edu.cn}
\affiliation{
\department{Gaoling School of Artificial Intelligence}
\institution{Renmin University of China}
  \city{Beijing}
  \country{China}
}

\author{Yutao Zhu}
\orcid{0000-0002-9432-3251}
\email{yutaozhu94@gmail.com}

\author{Ji-Rong Wen}
\orcid{0000-0002-9777-9676}
\email{jrwen@ruc.edu.cn}
\affiliation{
\department{Gaoling School of Artificial Intelligence}
\institution{Renmin University of China}
  \city{Beijing}
  \country{China}
}
\author{Ruibin Xiong}
\authornotemark[1]
%\orcid{0000-0002-9777-9676}
\email{xiongruibin18@mails.ucas.ac.cn}

\author{Mang Wang}
%\orcid{0000-0002-9777-9676}
\email{songmu@baichuan-inc.com}

\author{Weipeng Chen}
%\orcid{0000-0002-9777-9676}
\email{chenweipeng@baichuan-inc.com}
\affiliation{
\institution{Baichuan Intelligent Technology}
  \city{Beijing}
  \country{China}
}

%%
%% By default, the full list of authors will be used in the page
%% headers. Often, this list is too long, and will overlap
%% other information printed in the page headers. This command allows
%% the author to define a more concise list
%% of authors' names for this purpose.
%\renewcommand{\shortauthors}{Trovato et al.}

%%
%% The abstract is a short summary of the work to be presented in the
%% article.
\begin{abstract}
The outstanding capabilities of large language models (LLMs) render them a crucial component in various autonomous agent systems. While traditional methods depend on the inherent knowledge of LLMs without fine-tuning, more recent approaches have shifted toward the reinforcement learning strategy to further enhance agents' ability to solve complex interactive tasks with environments and tools. However, previous approaches are constrained by the sparse reward issue, where existing datasets solely provide a final scalar reward for each multi-step reasoning chain, potentially leading to ineffectiveness and inefficiency in policy learning. In this paper, we introduce \ours{}, which utilizes step-wise reward to optimize the agent's reinforcement learning process. Inheriting the spirit of novice-to-expert theory, we first compare the actions of the expert and the agent to automatically generate intermediate rewards for fine-grained optimization. Additionally, we propose implicit-reward and inverse reinforcement learning techniques to facilitate agent reflection and policy adjustment. Further theoretical analysis demonstrates that the action distribution of the agent can converge toward the expert action distribution over multiple training cycles. Experimental results across various datasets indicate that \ours{} outperforms existing baseline methods. 
\end{abstract}
%consolidation

%%
%% The code below is generated by the tool at http://dl.acm.org/ccs.cfm.
%% Please copy and paste the code instead of the example below.
%%

%%
%% Keywords. The author(s) should pick words that accurately describe
%% the work being presented. Separate the keywords with commas.
%% A "teaser" image appears between the author and affiliation
%% information and the body of the document, and typically spans the
%% page.

%\received{20 February 2007}
%\received[revised]{12 March 2009}
%\received[accepted]{5 June 2009}

%%
%% This command processes the author and affiliation and title
%% information and builds the first part of the formatted document.
\maketitle

\section{Introduction}
%虽然已有方法xxx但是xxx，所以我们要引入RL
Large language models (LLMs) have begun a revolutionary era in artificial general intelligence (AGI), due to their remarkable capabilities in handling complex interactive tasks with environments and tools~\cite{agentsurvey23,agentsurvey24}. The tasks involve multiple areas including web browsing~\cite{mind2web24nips}, web shopping~\cite{WebShop22nips}, house holding~\cite{ALFWorld21iclr}, and complex question answering~\cite{HotpotQA18emnlp,2wikimultihop20coling,musique21tacl}. 
Although these models (\eg, ChatGPT~\cite{openai2024gpt3.5technicalreport} and GPT-4~\cite{openai2024gpt4technicalreport}) are endowed with extensive knowledge during pre-training on a large-scale corpus, they demonstrate a tendency to generate hallucinated content~\cite{hallu20arxiv1,hallu20arxiv2}. 
To tackle this issue and further align with human preferences, researchers have introduced training LLM agents with reinforcement learning (RL) to enhance their ability for complicated task planning and resolving. 
%post-pre-training alignment with the more desirable behaviors

%Initial efforts in developing LLM agents~\cite{SFT22nips,SFT23emnlp} concentrated on behavior clone (BC)~\cite{BC91nc,BC08icml}, which maximizes the token-level generation probability of the expert behaviors through an auto-regressive loss. This method, while straightforward and reward-free, falls short when confronted with training data shortage and struggles to generalize beyond the training data distribution. Recognizing these constraints, researchers have shifted towards Reinforcement Learning from Human Feedback (RLHF)~\cite{RLHF17nips,RLHF19arxiv,RLHF20nips,RLHF22arxiv} which leverages manually annotated preferences as reward signals, encouraging high-scoring outputs and penalizing low-scoring ones. To further mitigate the dependence on extensive human-annotated data, recent work~\cite{SPIN24icml,ETO24acl} has introduced agent-generated trajectories into the training dataset. 

\begin{figure}
    \centering
    \includegraphics[width=0.92\linewidth]{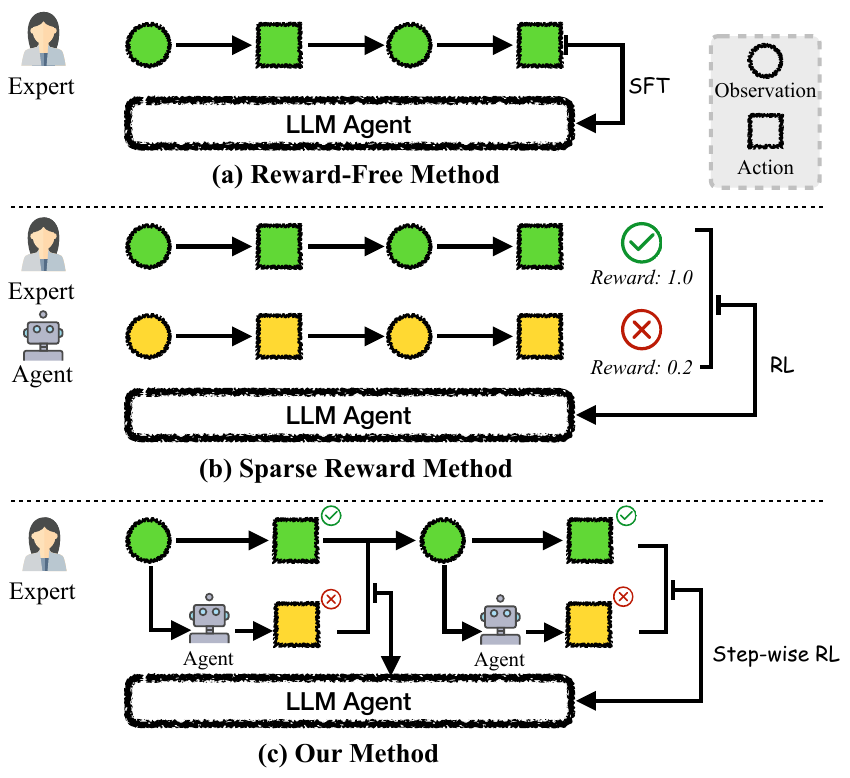}
    \caption{The comparison between our step-wise feedback LLM agent framework and previous approaches.}
    \label{fig:intro_model}
\end{figure}

Initial efforts in developing LLM agents~\cite{SFT22nips,SFT23emnlp,BC91nc,BC08icml} concentrated on maximizing the token-level generation probability of the expert actions, denoted in Figure~\ref{fig:intro_model}(a).
These methods, while straightforward and reward-free, fall short when confronted with the training data shortage situation and struggle to generalize beyond the training data distribution. Recognizing these constraints, researchers~\cite{RLHF17nips,RLHF19arxiv,RLHF20nips,RLHF22arxiv,ETO24acl} have shifted towards leveraging manually annotated preferences or the final environment feedback as additional reward signals and conducting reinforcement learning training on the basis of the supervised fine-tuning (SFT) model. 
%Recognizing these limitations, researchers~\cite{RLHF17nips,RLHF19arxiv,RLHF20nips,RLHF22arxiv,ETO24acl} have increasingly adopted manually annotated preferences and final environment feedback as supplementary reward signals. This approach allows for the continuation of reinforcement learning (RL) training based on models that have undergone supervised fine-tuning.
%encouraging high-scoring outputs and penalizing low-scoring ones. 
Nevertheless, these methods are restricted by the sparsity and delay of the reward signals. 
As shown in Figure~\ref{fig:intro_model}(b), existing reward signals are represented as a single scalar reward for each generated observation-action trajectory. Such sparse feedback renders it challenging for the model to discern the quality of each action, particularly for tasks with long reasoning chain. Consequently, the model struggles to precisely refine low-quality actions, resulting in low learning efficiency. The delayed reward feedback prevents the model from making timely corrections, potentially leading to sub-optimal responses. 

Process-supervised reinforcement learning~\cite{rewardtype22arxiv,rewardtype23arxiv} presents a promising solution to these challenges by providing supervision at each intermediate reasoning step. Throughout the process of agent reasoning, the reward of each intermediate step can assist in identifying underperformed policies timely, allowing for agent capability rapid improvements. 
In light of this, \textbf{we propose to optimize agent policy by incorporating step-wise supervision into reinforcement learning}.
However, directly applying step-wise supervision to LLM agents introduces its own set of challenges. First, value assessments for individual steps are often absent from the current multi-step agent interaction datasets, leaving only a final evaluation. Even for human annotators, fully understanding the contribution of each step to the ultimate outcome presents a significant challenge that can be both costly and labor-intensive. Furthermore, the necessity for agents to interact with the dynamically changing environment makes the situation even more complicated. 
Sampling reward distributions based on MCTS~\cite{MCTS12ieee} requires the agent to interact with the environment until obtaining the final reward which is non-parallelizable and inefficient. 

%Furthermore, the necessity for agents to engage with a dynamically changing environment adds an additional layer of complexity. Sampling reward distributions through Monte Carlo Tree Search (MCTS)~\cite{MCTS12ieee} necessitates that the agent interact with the environment until a final evaluation is obtained, which inherently limits the ability to parallelize this process.

% to evaluate the contribution of each intermediate reasoning step to the final outcome.
%Even when utilizing Monte Carlo Tree Search (MCTS) to sample reward distributions, agents are required to interact with the environment until producing final outcomes. These outcomes are then used to retroactively assign rewards to earlier reasoning steps. This iterative approach is non-parallelizable, and the vast sampling space presents a considerable challenge to efficiency.

%即使利用MCTS的方式采样reward分布，对于多步推理agent任务，agent需要在每个推理步骤与环境不断交互产生最终结果，并利用最终结果的reward反推步骤reward。这一过程无法并行，巨大的采样空间对效率是一个严峻的挑战。
%the dynamically changes of the environment leads to the situation even more complex and challenging to evaluate the quality of intermediate reasoning steps. 
%the need for many tasks to dynamically adjust the policy based on the environment feedback leads to the situation even more complicated and challenging to evaluate the quality of intermediate reasoning steps.
%To tackle this issue, we propose an step-wise reinforcement learning agent framework without human annotated intermediate rewards.  
%Our key idea is inspired by 

Considering the aforementioned concerns, we aim to efficiently construct step-wise reward supervision without additional human annotation to address the ability gap between the LLM agent and the expert. %, thereby learning a superior agent policy. 
%fully leverage the intermediate step of the expert data without additional reward annotation to effectively address the ability gap between the LLM agent and the expert, thereby learning a superior agent policy. 
We take inspiration from Benner's novice-to-expert theory~\cite{benner1982novice,benner1984novice}—novices can gradually align with expert policy through repeatedly observing expert behaviors with autonomous practicing and reflection of their current policy~\cite{benner1982novice,benner1984novice}. 
Intriguingly, even lacking explicit process-supervised reward signals, novices can still progressively approximate expert policy and respond swiftly to external stimuli. 
This cognitive proficiency mirrors the challenge of adapting step-wise reinforcement learning in agent interaction tasks—lacking step-wise supervision and flexibility. 
%inefficiency requirements. 
%. %Therefore, it stands to reason that imitating novice-to-expert could offer benefits for agent policy learning. 

%novices can eventually align with expert behavior. 
%By emulating such human learning process, we strive to train the LLM agent in a manner that allows it to progressively refine its capabilities. Consequently, the agent is able to enhance its performance by gradually absorbing expert experiences and adjust strategies. 
%, thereby eliminating the need for manual annotation.
%which emulate the process of novice-to-expert by automatically construct process supervision signals without manual annotation for step-wise reinforcement learning. 
%which aims to imitate the process of humans from novices to experts, and automatically construct process supervision signals without manual annotation. % to train LLM agents.
%emulate human observational learning ability in social cognitive theory to learn from other users’ history what content the current user has not visited but may be interested in, thereby enhancing the diversity of recommendations. 

Drawing on the above motivations, we propose a \textbf{step}-wise LLM \textbf{Agent} learning framework (\textbf{\ours{}}), which emulates the novice-to-expert learning process by automatically constructing supervision signals for step-wise reinforcement learning, thereby approaching the expert policy. 
We delineate the novice-to-expert process into two distinct steps in the context of agent tasks, including inspection and reflection. 
Specifically, for the \textbf{inspection} stage, our target is to recognize the policy distinction between the agent and the expert. We begin by observing expert behavior patterns and then force the agent to practice independently at each step. This facilitates a deeper and fine-grained comprehension of the expert's decision-making processes, spontaneously providing step-wise reward feedback. 
Next, we devise a \textbf{reflection} module to effectively adjust and improve agent policy based on the practice results. We devise two strategies for agent reflection, including implicit-reward reinforcement learning and inverse reinforcement learning. %The former utilizes pair-wise reward loss for optimization while the latter learns an explicit step-wise reward function. 
%由于agent task是多轮交互任务，细粒度的学习专家的行为有利于agent优化。因此，对于每个专家行为，我们都让agent产生对应的行为，细粒度的对比。
%discover users to be observed. Since the relationship between the observer and exemplar limits behaviors that can be learned, we carefully select the exemplar to ensure balancing potential interest exploration with the preservation of recommendation accuracy.  Next, considering the variety of user behaviors, we propose \textbf{purification} to extract valuable information in observed users' histories requiring special attention. This guarantees that the selected information better aligns with the user's interests, increasing learning efficiency. Moreover, we devise a \textbf{retention} module to effectively memorize the exemplars' behaviors. An attention-based knowledge transfer loss is conducted to disseminate the diversified information derived from exemplars to the current user’s memory. 
To validate the effectiveness of our model, we conduct extensive experiments on three different scenarios of agent interactive tasks. Experimental results consistently demonstrate that our model \ours{} outperforms the state-of-the-art LLM agent models. This clearly indicates the superiority of applying step-wise reward reinforcement learning to LLM agent policy learning. 

Our main contributions are three-fold:

(1) We propose a step-wise reinforcement learning framework \ours{} that automatically constructs intermediate feedback to progressively and efficiently optimize the agent policy to eventually align with the expert policy. 

%无人工标注构造训练数据
(2) We introduce two stages encompassing inspection and reflection, and construct process-supervised training data without human annotation to facilitate the novices becoming experts. 
%We construct process-supervised training data without human annotation, offering process-supervised training data for further reflection mechanism. 
%We integrate the external memory units with LLM to provide efficient access to extensive user histories.
%Through empirical analysis, we summarize the shortcomings of existing diversified recommendation methods and illustrate the potential of utilizing observational learning to enhance diversity with minimal costs on accuracy. 

%implicit&explicit
(3) We devise two reflection strategies for step-wise optimization, including implicit-reward and inverse reinforcement learning. 

\section{Related Work}

\subsection{LLMs as Agent}
Recently, the outstanding capabilities of large language models (LLMs) have led researchers to explore adopting these models as agent core controllers and constructing artificial intelligence (AI) agents. The development of existing agent systems can be roughly divided into two primary categories: prompt-based methods and fine-tuning-based methods. 

%can be roughly categorized into prompt-based methods and fine-tuning-based methods. 

\textbf{Prompt-based Methods.}
Prompt-based methods~\cite{AutoGPT,BabyAGI} focused on carefully designing the prompt and directly utilizing closed-source large language models, such as ChatGPT~\cite{openai2024gpt3.5technicalreport} or GPT-4~\cite{openai2024gpt4technicalreport}, for task planning and reasoning. Chain-of-Thought (CoT) prompting~\cite{COT22nips} was the fundamental of most prompt-based methods which introduced intermediate reasoning steps in demonstrations to enhance the capacity to do sophisticated reasoning. Inherit the spirit of CoT prompting, ReAct~\cite{ReAct23iclr} devised a think-and-act format prompt to inspire LLMs to generate both reasoning traces and task-specific actions in an interleaved manner. ToT~\cite{ToT24nips} further generalized to tree-structure ensuring to explore various reasoning paths and make global decisions by looking ahead or backtracking when necessary. %deliberate reasoning process. 
Driven by human revision behavior, SELF-REFINE~\cite{selfrefine24nips} utilized a single LLM as the generator, refiner, and feedback provider. 
In addition, Reflexion~\cite{reflexion24nips} leveraged linguistic feedback maintained in a memory buffer to reinforce agents and induce better decision-making. 

\textbf{Fine-tuning-based Methods.}
Although prompt-based methods could achieve promising performances without training, they heavily rely on well-designed prompts and advanced closed-source models (\eg, ChatGPT and GPT-4) leading to high usage costs. To address these challenges, recent studies~\cite{Fireact23arxiv,AgentTuning23arxiv,AgentFLAN24arxiv,AgentLumos24arxiv} constructed expert trajectory data with teacher agents (\eg, GPT-4 or humans) and performed supervised fine-tuning on open-source LLMs (\eg, LLaMA~\cite{llama23arxiv} and Mistral~\cite{mistral23arxiv}). Taking a step further, NAT~\cite{NAT24arxiv} and ETO~\cite{ETO24acl} introduced negative samples during the fine-tuning to reduce model hallucinations and enhance robustness. Furthermore, Rejection sampling Fine-Tuning (RFT)~\cite{RFT23arxiv} collected correct reasoning paths generated by the supervised model to enrich fine-tuning datasets while SPIN~\cite{SPIN24icml} empowered a weak AI agent leveraging its generated data for training without additional human annotation. 

In this paper, we focus on fine-tuning LLMs with reinforcement learning and devise a step-wise learning strategy to align the capabilities of the agent with the expert.

%Existing methods only use the final reward signal to optimize the large language model.
%are sensitive to the quality of the prompt and bridge this gap

\subsection{Reinforcement Learning for LLMs}
With the development of the LLMs, reinforcement learning (RL)~\cite{RLHF17nips, RLHF19arxiv} plays a vital role in improving the capabilities of LLMs. 
Actor-Critic~\cite{actor-critic99nips} was the basis of many advanced RL algorithms which leveraged the actor policy network to interact with the environment and perform policy updates under the guidance of the critic value function. Based on the actor-critic algorithm, Trust Region Policy Optimization (TRPO)~\cite{TRPO15icml} introduced trust region to ensure monotonic performance of policy learning while Proximal Policy Optimization (PPO)~\cite{PPO17arxiv} further proposed penalty and clip strategies to simplify the algorithm implementation. 
To solve the problem of instability during Reinforcement Learning from Human Feedback (RLHF) training, Direct Preference Optimization (DPO)~\cite{DPO23nips} adopted a simple classification loss to fine-tuning LLMs and achieve higher efficiency and better performances. Since the reward signal is uncertain or sparse in real-world scenarios, researchers proposed behavior cloning (BC)~\cite{BC08icml} to imitate the behaviors of experts. Furthermore, Generative Adversarial Imitation Learning (GAIL)~\cite{GAIL16nips} devised an iterative reward function learning strategy forcing the agent to fit the expert data distribution. 

The reward function in previous agent approaches was either manually annotated~\cite{RLHF17nips,RLHF22arxiv,RLHF20nips} or limited to the final reward feedback from the environment~\cite{WebShop22nips,ALFWorld21iclr,ScienceWorld22emnlp}. In this work, we propose a step-wise reinforcement learning method and automatically generate rewards for each step.

\section{Preliminaries}

In this section, we first formulate the agent task and then review supervised fine-tuning for LLMs, a crucial step before reinforcement learning that prepares the model for specific tasks. 
%before reinforcement learning that ensures the model is well-prepared to perform specific tasks.

\subsection{Problem Formulation}

The process of an agent interacting with the environment for task solving can be formalized as a partially observable Markov decision process (POMDP) with the state set $\mathcal{S}$, action set $\mathcal{A}$, observation set $\mathcal{O}$, transition function $\mathcal{F}: \mathcal{S}\times\mathcal{A}\rightarrow\mathcal{S}$, and reward function $\mathcal{R}: \mathcal{S}\times \mathcal{A}\rightarrow [0,1]$. Initially, the environment provides a general task instruction \ttt{Prompt}$_\text{sys}$ as the system prompt, along with the agent's initial observation $o_1\in\gO$ as the specific task input, and the agent needs to interact with the environment multiple times for completing the task and generating responses. 

Specifically, at the time step $t$, the large language model agent parameterized by $\theta$ receives an observation $o_t\in\gO$ from the environment and decides to take an action $a_t\in\gA$ according to the policy $\pi_\theta(\cdot|s_t)$, where $s_t=($\ttt{Prompt}$_\text{sys},o_1, a_1, \cdots, a_{t-1}, o_{t})\in\gS$ is the current state of the environment. The interaction process repeats until the task completes or exceeds the maximum steps. A reward $r\in[0,1]$ is then computed for the final trajectory $($\ttt{Prompt}$_\text{sys}, o_1, a_1, \cdots, o_{n}, a_n)$, where $r=1$ indicates the task is success and $0$ means failure.\footnote{We omit \ttt{Prompt}$_\text{sys}$ for simplification in the following expressions.} The conditional probability distribution for the overall process $\pi_\theta(a_n|o_1)$ can be denoted through a decomposition as follows:
\begin{align}
    \pi_\theta(a_n|o_1)=\prod_{t=1}^n \pi_\theta(a_t|s_t). \label{eq:MDP}
\end{align}

%In the following section, we 

\subsection{Supervised Fine-tuning}
Supervised fine-tuning (SFT) entails leveraging relatively smaller labeled expert data to better adapt the pre-trained LLMs to specific domains or downstream tasks~\cite{SFT22nips,SFT23emnlp}, providing a solid foundation for creating a powerful agent. %In this context, given a general task instruction \ttt{Prompt}$_\text{sys}$ and a specific task description as the prompts, the agent generate the trajectory $t_a$ following ReAct~\cite{ReAct23iclr}-form, which additionally generates Chain-of-Thought (CoT) rationales~\cite{COT22nips} before each action. 

Given an expert interaction trajectory $t_e=(\hat{o}_1, \hat{a}_1, \cdots, \hat{o}_{n}, \hat{a}_n)$ in the expert trajectory set $\gT$, we leverage the auto-regressive loss to fine-tune the initial LLM and obtain the base agent $\pi_{\theta_0}$ as follows:
\begin{align}
    L_\text{SFT}=-\mathbb{E}_{t_e\sim\gT}[\pi_\theta(\hat{a}_n|\hat{o}_1)]. \label{eq:SFT}
\end{align}
Following Equation~(\ref{eq:MDP}), $\pi_\theta(\hat{a}_n|\hat{o}_1)=\prod_{t=1}^n \pi_\theta(\hat{a}_t|\hat{s}_t)$, where $\hat{s}_t=(\hat{o}_1, \hat{a}_1, ..., \hat{o}_{t})$. We first concatenate the instruction prompt, actions and observations in trajectory $t_e$ as a token sequence $w=(w_1, ..., w_l)$ with length $l$. Then, the probability $\pi_\theta(\hat{a}_n|\hat{o}_1)$ in Equation~(\ref{eq:SFT}) can be formulated as follows:
\begin{align}
    \pi_\theta(\hat{a}_n|\hat{o}_1)=-\sum_k \log\pi_\theta(w_k|w_{<k}) \times \mathbf{1}_{w_k\in \gA},
\end{align}
where $w_{<k}$ indicates tokens before the $k$-th token and $\mathbf{1}_{w_k\in \gA}$ is an indicator function indicating whether $w_k$ is a token of actions generated by the agent. We mask the observation tokens and compute the probability solely for the action tokens. 
%directly computing the probability of actions with observation tokens masked

%\section{Self-Regulated Learning}
\section{From Novice to Expert}
\label{sec:method}
Large language model (LLM) agents have demonstrated superior capabilities in tackling complex interactive tasks, by leveraging reinforcement learning strategy to align the agent policy with human preferences. However, existing research on LLM agents~\cite{ETO24acl,SPIN24icml} encounter significant challenges stemming from reward signal sparsity and the complexities associated with reasoning process.
To address these limitations, in this section, we introduce a step-wise reinforcement learning framework to optimize the agent policy without manually annotating the procedural rewards. 
Our approach is inspired by the principles of Benner's novice to expert~\cite{benner1982novice,benner1984novice}, facilitating progressively self-iterative experience acquisition. By constantly monitoring the expert's behaviors and practice spontaneously, the LLM agent can accumulate experience and eventually advance from novice to expert proficiency.
%By identifying weaknesses in reasoning process, 

%enhancing cognitive level through inspection and reflection. 
%enhance the accuracy of answer derivation.

%a framework empowered by large language models with an efficient memory mechanism to enhance user modeling in personalized search.
%Motivated by this observation, in this section, we introduces a metacognitive retrieval-augmented generation framework. 

The overall framework of \ours{} is depicted in Figure ~\ref{fig:model}. \ours{} comprises two major phases: (1) \textbf{Inspection} and (2) \textbf{Reflection}. The details of the two stages are introduced in the following sections. 
%encompasses, comprises
%In this section, we introduce our proposed Expert-Novice Theory learning framework for tackling complex agent tasks. 

\begin{figure*}
    \centering
    \includegraphics[width=\textwidth]{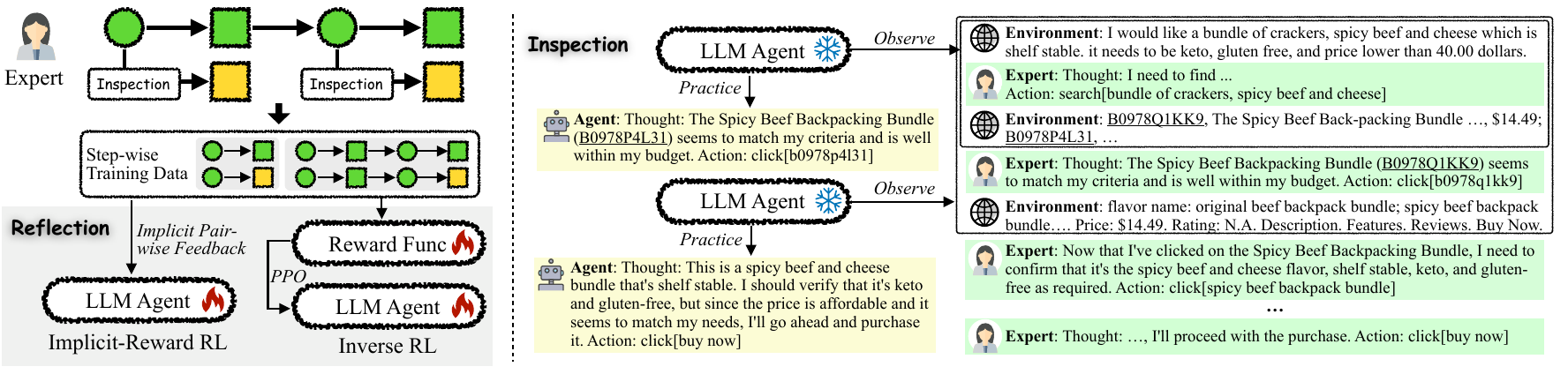}
    \caption{The architecture of our proposed framework \ours{} containing two stages: inspection and reflection. Blue snowfake indicates frozen parameters while red flame means trainable parameters. The example comes from the WebShop dataset.}
    \label{fig:model}
\end{figure*}

%\subsection{Model Architecture}

%学生（1）产生自己的答案（2）对答案（3）更新策略向标准答案靠拢的过程如何使用自我调节学习的理论解释
%有一个专家轨迹，对于每一个专家步骤，我都会产生对应的agent轨迹，将专家轨迹作为标准，更新agent策略使得生成的新轨迹和专家数据更像
%\subsection{Goal Establishing/Observation}
\subsection{Inspection: Recognizing Capability Gaps}
\label{subsec:inspection}
%自我调节学习的首要步骤是确定想要达到的学习目标。对于复杂的agent任务，LLMs agent需要不断与环境交互，反复尝试，逐步优化策略。由于现有数据集只有最终的奖励信号（例如任务的成功与否或者和正确答案的exact match），模型在每一步做出决策时，无法知道其行为对最终结果的具体贡献，导致模型优化效率低下。与此同时，模型需要更多的训练样本才能有效地学到如何在中间步骤做出正确决策。这增加了训练时间和资源的消耗。此外，稀疏的奖励信号使得模型需要经历完整的推理链条才能收到一次反馈，导致模型无法在推理过程中获得逐步反馈进行调整，学习过程缓慢且不稳定。为了解决上述问题，我们在不引入额外的标注数据的情况下设定细粒度的学习目标以期望agent的行为可以和专家行为对齐。
%在不引入额外的标注数据的情况下充分利用现有SFT数据，提升优化效率，快速优化模型。

%step-wise学习对由新手过渡到专家很重要，但已有数据中不包含这类信息，我们设计了一个无需人工标注的数据构造方法，self-play的构造step-wise数据,为了加快数据构造速度，且避免人工标注奖励信号
%新手成长为专家的第一步是观察专家的行为并进行独立尝试，从而认识到自己和专家的能力差距，帮助后续有针对性的改进策略

%Inspection, in accordance with Benner's novice to expert theory~\cite{benner1982novice,benner1984novice}, involves the novice initially observing expert behaviors and subsequently practicing the observed skills independently. 

Inspection, in accordance with Benner's novice to expert theory~\cite{benner1982novice,benner1984novice}, involves the novice initially observing expert behaviors and attempting to replicate these behaviors independently under the same circumstance. This comparative practice aims to recognize the capability gap between the novice and the expert, thereby facilitating subsequently novice policy improvements. 
Previous methods for constructing LLM agents~\cite{ETO24acl,SPIN24icml} focus on observing and imitating the complete behavior trajectory of the expert with the final environmental reward feedback for optimization. 
However, due to the complexity of the agent tasks, LLM agents need to constantly interact with the environment and engage in trial-and-error to arrive at the ultimate reasoning outcome. The inherent multi-step reasoning characteristics of agent tasks bring dual challenges of efficiency and effectiveness for the novice's self-attempts of the complete trajectory. First, emulating the full trajectory of the expert and acquiring the final environmental feedback require the agent to constantly interact with the environment. This interaction is sequential and cannot be parallelized, resulting in the significant consumption of computational time and resources. Besides, the necessity for the novice to comprehend every expert action simultaneously can lead to information overload. This overload complicates the novice to digest and master the specifics of each behavior, often resulting in inefficient learning processes. Consequently, novices may require additional training data or iterations to fully grasp the insights derived from the expert's experiences.

To address these limitations, it is essential for the novice to attentively observe and imitate the expert's actions step-by-step. This enables the novice to identify shortcomings in their behaviors and facilitate the mastery of critical skills. 
Specifically, considering an expert trajectory $t_e=(\hat{o}_1$, $\hat{a}_1$, $\cdots$, $\hat{o}_{n}, \hat{a}_n)$ with $n$-steps, we segment this trajectory after each action, treating each action as a short-term learning objective for the novice:
\begin{align}
    (\hat{o}_1, \hat{a}_1, \cdots, \hat{o}_{i}, \hat{a}_i)\in\gT_\text{sample},\quad i=1, 2, \cdots, n.
\end{align}
%where $i$ is from 1 to $n$, $\hat{o}_{i}$ and $\hat{a}_{i}$ is the observation and the action of the expert, respectively. 

%we prompt the agent to generate corresponding action at each step. For example, for the $t$-step of the expert trajectory, it is in state $s=(o_1, a_1, ..., a_{t-1}, o_t)$. 
%通过亲身练习发现自己与专家行为的差距，在反复练习中总结经验并形成自己的行为模式

When the novice establishes learning targets, it triggers the practice stage in the expert-novice learning process. This spontaneous exercise is geared towards identifying the behavioral discrepancies between the novice agent and the expert, allowing for the accumulation of experience and the gradual development of the novice's behavioral patterns through repeated practice.
Central to this spontaneous exercise is that the novice generates actions based on the previously established learning targets. 
Specifically, for each learning objective in $\gT_\text{sample}$, we treat the state $\hat{s}_i=(\hat{o}_1, \hat{a}_1, \cdots, \hat{o}_{i})$ as the prompt and let the agent $\pi_{\theta}$ parameterized by $\theta$ to generate the appropriate action as Equation~(\ref{eq:atheta}) and obtain the corresponding agent trajectory $(\hat{o}_1, \hat{a}_1, \cdots, \hat{o}_{i}, a^{\theta}_i)\in\gT^\theta_\text{sample}$.
\begin{align}
    \label{eq:atheta}
    a^{\theta}_i \sim \pi_\theta(a|s).
\end{align}

%We treat the expert trajectory $t_\text{sample}$ as the chosen trajectory while the corresponding agent trajectory $(\hat{o}_1, \hat{a}_1, \cdots, \hat{o}_{n-1}, a^{\theta}_t)$ as the rejected one. 

%对于每个学习目标，学生会产生对应的答案，

%\subsection{Monitoring}

\subsection{Reflection: Strategizing Policy Refinement}
In novice-to-expert theory, progression toward expert-level performance requires novices to reflect on their interaction trajectories. This introspection is intended to summarize and internalize experiences, ultimately leading to the development of individualized behavior patterns and policies. Therefore, in this section, we leverage interactions constructed in Section~\ref{subsec:inspection} and devise two distinct reflection strategies, including implicit-reward reinforcement learning and inverse reinforcement learning. 

%compare their behaviors and improve the policy. In Section~\ref{subsec:inspection}, we construct comparative data on novice and expert behaviors. 

%it triggers the process of reflection and adjustment. This introspective exercise is intended to summarize and internalize experiences, ultimately leading to the development of individualized behavior patterns. 
%the completion of a task triggers the self-regulated process of reflection and adjustment. This introspective exercise is geared towards identifying the weaknesses of the provided response and revise the strategy to align with the desired behaviors. 
%We employ two types of reflection strategies including implicit-reward reinforcement learning and inverse reinforcement learning. 

\subsubsection{Implicit-Reward Reinforcement Learning}
\label{subsubsec:implicit}
We begin by directly comparing the actions of the expert and the novice agent without introducing explicit reward estimation. Given a trajectory pair $(t_\text{sample}, t_{\theta})$ where $t_\text{sample}=(\hat{o}_1$, $\hat{a}_1, \cdots, \hat{o}_{i}, \hat{a}_i)$ is the expert trajectory while $t_{\theta}=(\hat{o}_1$, $\hat{a}_1, \cdots, \hat{o}_{i}, a^{\theta}_i)$ is the corresponding agent trajectory. Inheriting the spirit of previous works~\cite{ETO24acl,SPIN24icml}, we utilize the direct preference optimization loss~\cite{DPO23nips}, defined as follows:
\begin{equation}
    \begin{aligned}
    L_\text{implicit}(\pi_\theta, \pi_\text{ref}) = -\mathbb{E}[\log\sigma(\beta\log\frac{\pi_\theta(\hat{a}_i|\hat{s}_i)}{\pi_\text{ref}(\hat{a}_i|\hat{s}_i)}-\beta\log\frac{\pi_\theta(a^{\theta}_i|\hat{s}_i)}{\pi_\text{ref}(a^{\theta}_i|\hat{s}_i)})],
    \label{eq:dpo}
\end{aligned}
\end{equation}
where $\pi_\theta$ is the current agent policy needed to be optimized, $\pi_\text{ref}$ is the reference model initialized with the agent policy and $\beta$ is a hyper-parameter.

\subsubsection{Inverse Reinforcement Learning}
%GAIL Adversarial learning
\label{subsubsec:irl}
Considering the lack of reward signals for each reasoning step in existing datasets, we introduce an inverse reinforcement learning (IRL) method~\cite{IRL00icml,IRL98,IRL21AI,GAIL16nips}. This method first infers the step-wise reward function based on the expert's and agent's behaviors and then leverages the reward function to fine-grained optimizes the agent policy.

%\begin{definition}
We first define the occupancy measure $\rho_\pi$ for a policy $\pi$, indicating the normalized distribution of state-action pairs when the agent adopts policy $\pi$ to explore the environment:
\begin{align}
    \label{rhopi}
    \rho_\pi(s,a)=(1-\gamma)\sum_{t=0}^\infty \gamma^t P_\pi(s_t=s) \pi(a|s),
\end{align}
where $1-\gamma$ is the normalization factor, $P_\pi(s_t=s)$ represents the probability of the agent in state $s$ at time $t$ when adopting policy $\pi$. 
%\end{definition}

To accurately imitate the expert policy, it is essential to ensure that the policy distribution generated by the agent is as similar as possible to that generated by the expert. This can be achieved by maintaining that the agent's occupancy measure $\rho_{\pi_\theta}$ is as close as possible to that of the expert $\rho_{\pi_e}$. We adopt Jensen-Shannon divergence (JS) to measure the distance between two distributions. 
\begin{align}
    \min_\pi \text{JS}(\rho_{\pi_\theta}, \rho_{\pi_e})-\lambda H(\pi_\theta),
    \label{eq:js}
\end{align}
where $\lambda$ is the hyper-parameter, $H(\pi_\theta)\overset{\triangle}{=}\mathbb{E}_{\pi_\theta}[-\log \pi_\theta(a|s)]$ is the $\gamma$-discounted causal entropy~\cite{Hpi14ieee} of the agent policy.  
%$\text{JS}(\rho_{\pi_\theta}, \rho_{\pi_e})$ is the Jensen-Shannon divergence between the probability distribution $\rho_{\pi_\theta}$ and $\rho_{\pi_e}$. 

%\overline{\mathbb{R}}
Following GAIL\cite{GAIL16nips}, the Jensen-Shannon divergence $\text{JS}(\rho_{\pi_\theta}, \rho_{\pi_e})$ and be represented by a convex cost function regularizer $\varpi(\rho_{\pi_\theta}-\rho_{\pi_e})$, up to a constant shift and scaling. The definition of the convex cost function regularizer $\varpi: \mathbb{R}^{\gS\times\gA}\to\mathbb{R}\cup\{\infty\}$ is defined as:
\begin{align}
    \varpi(c)\triangleq\left\{\begin{array}{ll}
       \mathbb{E}_{\pi_e}[-c(s,a)-\log(1-e^{c(s,a)})]  &  c<0; \\
        +\infty & c\geq 0.
    \end{array} \right.
\end{align}

According to \cite{GAIL16nips}, the optimal solution of the above regularizer $\varpi(\rho_{\pi_\theta}-\rho_{\pi_e})$ is denoted as follows:
\begin{align}
    \sup_{D\in(0,1)^{\mathcal{S}\times\mathcal{A}}} \mathbb{E}_{\pi_\theta}[\log(D(s,a))]
    +\mathbb{E}_{\pi_e}[\log(1-D(s,a))]. \notag
\end{align}

Therefore, the optimization problem of Equation~(\ref{eq:js}) can be transformed into finding a saddle point $(\pi, D)$ of the below Equation:
\begin{align}
    \mathbb{E}_{\pi_\theta}[\log(D(s,a))]+\mathbb{E}_{\pi_e}[\log(1-D(s,a)]-\lambda H({\pi_\theta}).
    \label{eq:discriminator}
\end{align}

We directly train a discriminator network $D: \gS\times\gA\rightarrow(0,1)$, utilizing data sampled from the expert and agent trajectories. The primary objective of $D$ is to differentiate between the distribution of data generated by the agent policy $\pi_\theta$ and the expert policy $\pi_e$. When $D$ cannot distinguish data generated by the agent from the expert, then the occupancy measure of the agent $\rho_\pi$ has successfully matched that of the expert $\rho_{\pi_e}$. 
The discriminator network $D$ can be interpreted as an implicit reward model providing step-wise learning signals to the agent policy. The complete learning process of \ours{}-inverse is introduced in Algorithm~\ref{Alg-IRL}. 

\begin{algorithm}[!t]
    \caption{\ours{} with Inverse Reinforcement Learning}\label{Alg-IRL}
    %\hspace*{0.01in} 
    \begin{algorithmic}[1]
    \State \textbf{Input:} Expert trajectories $(\hat{o}_1, \hat{a}_1, ..., \hat{o}_{n-1}, \hat{a}_{n})\in\gT$, \\
    \qquad\quad agent policy initialized by $\pi_{\theta_0}$
    %\hspace*{0.02in} 
    \State \textbf{Output:} Final agent policy $\pi_\theta$
    %\REQUIRE{Expert trajectories $\gD\in (o_1, a_1, ..., o_{n-1}, a_{n})$}
    %\ENSURE{Final policy $\pi_\theta$}
    \State Initialize $\pi_{\theta_1}\leftarrow\pi_{\theta_0}$
    \For{iteration $i=1,2, ...$}
    \State \textcolor{grey}{// Inspection Stage.}
    \State For each sampled step-wise expert trajectory $(\hat{o}_1, \hat{a}_1, \cdots, $ $\hat{o}_{t}, \hat{a}_t)\in\gT_\text{sample}$ generate the corresponding agent trajectory $(\hat{o}_1, \hat{a}_1, \cdots, \hat{o}_t, a^{\theta}_t)\in\gT^\theta_\text{sample}$ with policy $\pi_{\theta_i}$
    \State \textcolor{grey}{// Reflection Stage.} 
    \For{data in $(\gT_\text{sample}, \gT^\theta_\text{sample})$}
        \State train the discriminator with the following loss:
        \begin{equation}
        \small
        \begin{aligned}
            \mathbb{E}_{\pi_\theta}[\log(D_w(s,a))]+\mathbb{E}_{\pi_e}[\log(1-D_w(s,a)]
        \end{aligned}
        \end{equation}
        \State Update the parameter of the discriminator $D_w\rightarrow D_{w'}$ 
        \State Take a policy step with PPO rule and reward function $\log(D_{w'}(s,a))$ and update policy $\pi_{\theta_i}\rightarrow\pi_{{\theta_i}'}$.
    \EndFor
    \EndFor
    \end{algorithmic}
\end{algorithm}

%The optimization function for maximum causal entropy IRL is defined as:
%\begin{align}
%    \max_{c\in\gC}\Big(\min_{\pi\in\Pi}-H(\pi)+\mathbb{E}_\pi[c(s,a)]\Big)-\mathbb{E}_{\pi_e}[c(s,a)]-\varpi(c)
%\end{align}

\section{Theoretical Analysis}
%我学习到的行为分布和专家分布经过多次迭代可以覆盖
In this section, we provide a theoretical analysis to prove that the distribution of actions generated by the agent can converge toward the expert action distribution over multiple training cycles. 

\begin{assumption}
    \textit{The loss function of Equation~(\ref{eq:dpo}) and (\ref{eq:discriminator}) is bounded and Lipschitz continuous.} 
    \label{assumption}
\end{assumption}
Since our policy update method employs gradient descent, under Assumption~\ref{assumption}, the policy \(\pi_\theta\) will converge to a local minimum as the iterations increase. The following analyses are conducted under Assumption~\ref{assumption}. 

\begin{proposition}
    \textit{The occupancy measure \(\rho_{\pi_\theta}\) for the agent policy can converge to closely approximate the expert's occupancy measure \(\rho_{\pi_e}\), after several iterations. }
    \label{proposition:rho}
\end{proposition}
\begin{proof}
    The occupancy measure represents the normalized distribution of state-action pairs. Consequently, the discrepancy between \(\rho_{\pi_\theta}\) and \(\rho_{\pi_e}\) can be measured using the Kullback-Leibler or Jensen-Shannon divergence
    \(
    \text{KL/JS}(\rho_{\pi_\theta},\rho_{\pi_e})
    \). In this context, Proposition~\ref{proposition:rho} can be reformulated into Proposition~\ref{proposition:main}. 
\end{proof}

\begin{proposition}
    \label{proposition:main}
    \textit{Proving that optimizing the loss function can be ultimately equivalent to the minimized KL/JS divergence. }
\end{proposition}
    In the remaining section, we demonstrate that this proposition is valid for both reflection mechanisms, including implicit-reward reinforcement learning and inverse reinforcement learning. 
\begin{proof}
    In the following parts, we first prove that Proposition~\ref{proposition:main} holds for implicit-reward reinforcement learning, and then prove for inverse reinforcement learning optimization. 
    
    \textbf{STEP 1.} 
    %We first prove that for implicit-reward reinforcement learning optimization, Proposition~\ref{proposition:rho} holds. 
    According to \citet{DPO23nips}, the optimal solution of the KL-constrained reward maximization objective can be rearranged so that the reward function can be expressed as
    \[
    r(s,a)=\beta \log \frac{\pi_\theta(a|s)}{\pi_\text{ref}(a|s)}+\beta \log Z(s),
    \]
    where \(Z(s)\) is the partition function~\cite{partition23arxiv,partitionfunc22nipz}. 
    Following Bradley-Terry model~\cite{BT22pmlr}, we have :
    \[
    p(a_1>a_2|s)=\sigma(r(s,a_1)-r(s, a_2)).
    \]
    Then, the policy objective can be formulated as Equation~(\ref{eq:dpo}) which is equivalent to minimizing the KL divergence. 
    
    \textbf{STEP 2.} Inverse reinforcement learning first trains a discriminator network, which subsequently generates scores that serve as the reward function for optimizing the policy network. Its optimization target can be denoted as:
    \[
    J(\theta) = -\mathbb{E}_{(s,a) \sim \pi_\theta} [\log D(s,a)].
    \]
    
    According to the policy gradient theorem, the gradient of \(J(\theta)\) can be expressed as:
    \[
    \nabla_\theta J(\theta) = \mathbb{E}_{(s,a) \sim \pi_\theta} \left[ \nabla_\theta \log \pi_\theta(a|s) R(s,a) \right].
    \]
    We utilize the output of the discriminator as the reward and the gradient of the policy becomes: 
    \[
    \nabla_\theta J(\theta) = \mathbb{E}_{(s,a) \sim \pi_\theta} \left[ \nabla_\theta \log \pi_\theta(a|s) D(s,a) \right].
    \]
    The expected return of the policy update (\ie, the output of the discriminator) is related to the gradient of the policy parameters. 
    
    According to Equation~(\ref{eq:js}-\ref{eq:discriminator}), optimizing the loss function of the discriminator network is equivalent to reducing the JS divergence between the two occupancy measures. %Therefore, if the discriminator can effectively learn the difference between the trajectories generated by the expert and the agent, then after multiple iterations, the agent's policy will be gradually adjusted, and the distribution of the expert data and the agent data will gradually approach.
    The application of the policy gradient theorem enables the agent to optimize its strategy using feedback from the discriminator. This process ensures the trajectory generated by the agent to gradually approach the expert's trajectory distribution by maximizing the output of the discriminator. 
    
    %the optimization target of the inverse reinforcement learning method, is first minimizing the JS-divergence and the discriminator will gradually learn the ability to distinguish between expert trajectories and agent-generated trajectories. Then, update the strategy based on the reward given by the discriminator. 
\end{proof}

\iffalse
\begin{definition}
Consider an expert trajectory denoted as \( t_e = (s_i, \hat{a}_i) \), where \( s_i = (\hat{o}_1, \hat{a}_1, \ldots, \hat{o}_i) \) represents the sequence of states and actions up to the \( i \)-th step, with \( \hat{a}_i \) being the action selected by the expert at state \( s_i \). 
We define the expert preference function \( P(\hat{a}_i | s_i) \), which quantifies the probability of the expert selecting action \( \hat{a}_i \) given the current state \( s_i \). 
\end{definition}
This probabilistic \( P(\hat{a}_i | s_i)\) encapsulates the expert's decision-making likelihood under the state space. %Similarly, let the agent trajectory be represented as \( t_\theta = (s_i, a^\theta_i) \), where \( a^\theta_i \) is the action chosen by the agent according to its policy \( \pi_\theta \).

\begin{proposition}
    \label{proof:dpo}
    The agent strategy \( \pi_\theta(\hat{a}_i|s_i) \) can converge to closely approximate the expert preference function \( P(\hat{a}_i | s_i) \), after several reinforcement learning iterations of the agent strategy. 
\end{proposition}
\fi

\section{Experimental Settings}
\subsection{Datasets}
To thoroughly evaluate the ability of our proposed model \ours{}, we utilize representative tasks from three aspects, including web tasks, agent tasks, and multi-hop question-answering tasks. The statistics of these datasets are delineated in Table~\ref{tab:datastat}. 

\textbf{Web tasks} consist of WebShop~\cite{WebShop22nips} for online shopping and Mind2Web~\cite{mind2web24nips} for complex tasks on various websites. Rewards in the two datasets are dense variable and range from 0 to 1. 
%The two datasets provide dense final rewards ranging from 0 to 1. 

\textbf{Agent tasks} contain Science World~\cite{ScienceWorld22emnlp} for science experiments, and ALFWorld~\cite{ALFWorld21iclr} for embodied housework. The former contains continuous final rewards from zero to one while the latter has binary rewards demonstrating the completion of the task. For both datasets, we treat the in-distribution test sets as the validation set and the out-of-distribution unseen variations which aim to assess the generalization capabilities of agents as the test set.  

\textbf{Multi-hop question-answering tasks} include HotpotQA~\cite{HotpotQA18emnlp}, 2WikiMultihopQA~\cite{2wikimultihop20coling}, and MuSiQue~\cite{musique21tacl}. For each dataset, we leverage their associated Wikipedia articles contexts as our retrieval corpus to conduct multi-step reasoning. Considering the restrictions of experimental costs, following previous approaches~\cite{metacognitive24www,ReAct23iclr}, we utilize a subset of the entire dataset, selecting 5,000 samples for training from the training set and 500 samples each for the validation and test sets from the development set.
%To accelerate model training and evaluation, 

%We use questions in the development set for evaluation. 

%HotpotQA\footnote{\url{https://hotpotqa.github.io/}}2WikiMultihopQA\footnote{\url{https://github.com/Alab-NII/2wikimultihop?tab=readme-ov-file}} MuSiQue\footnote{\url{https://github.com/StonyBrookNLP/musique}}

\begin{table}
  \caption{Statistic of datasets in our experiments. } 
  \small
  \centering
  \label{tab:datastat}
  \begin{tabular}{llrrrr}
    \toprule
    \textbf{Type}&\textbf{Dataset} & \textbf{\# Train} & \textbf{\# Dev} & \textbf{\# Test} & \textbf{\#Turns} \\
    \midrule
    %\multicolumn{5}{c}{\textit{Agent Tasks}}\\
    \multirow{2}{*}{Web} 
    &WebShop & 1,938 & - & 200 & 4.9 \\
    &Mind2Web & 1,009& - & 912 & 7.3\\
    \midrule
    \multirow{2}{*}{Agent}
    &Science World & 1,483 & 194 & 241 & 14.4 \\
    &ALFWorld & 3,321 & 140 & 134 & 10.1 \\
    \midrule
    %\multicolumn{5}{c}{\textit{Question-Answering Tasks}}\\
    \multirow{3}{*}{\makecell[l]{Multihop\\QA}}
    &HotpotQA & 90,447 & 7,405 & 7,405 & 7.0 \\
    &2WikiMultihopQA & 167,454 & 12,576 & 12,576 & 8.2\\
    &MuSiQue & 19,938 & 2,417 & 2,417 & 7.8\\
    %5000 & 500 & 500 \\
  \bottomrule
\end{tabular}
\end{table}

\subsection{Backbone Models and Baselines}
We verify the effectiveness and robustness of our \ours{} on two widely-used open-source models: \textbf{Mistral-7B} (Mistral-7B-Instruct-v0.1) and \textbf{Llama-3-8B} (Meta-Llama-3-8B-Instruct). 

We compare the two variants (\ie, implicit and inverse) of our method \ours{} with several baselines including 
(1) Supervised Fine-Tuning (\textbf{SFT})~\cite{Fireact23arxiv,AgentTuning23arxiv} conducts behavioral cloning on expert trajectories, which is the base agent for \ours{} and other baselines. 
(2) Proximal Policy Optimization (\textbf{PPO})~\cite{PPO17arxiv} and Direct Preference Optimization (\textbf{DPO})~\cite{DPO23nips} are two representative reinforcement learning methods. We utilize the final task reward from the environment as the reward feedback for PPO. As for DPO, we adopt the trajectories generated by the agent as negative samples. 
(3) Rejection sampling Fine-Tuning (\textbf{RFT})~\cite{RFT23arxiv} and \textbf{SPIN}~\cite{SPIN24icml} incorporate the success trajectories of the agent to the expert trajectory dataset and trains the agent on new augmented trajectories.
(4) \textbf{NAT}~\cite{NAT24arxiv} and \textbf{ETO}~\cite{ETO24acl} introduce rejected trajectories into the training process, allowing the agent to learn from its failure experiences. 
%(3) Best-of-N sampling employs SFT base agent and selects the trajectory with the best reward within N samplings. Here we set N to 10.
%(4) Proximal Policy Optimization (PPO)~\cite{PPO17arxiv} is an RL method directly optimizing the SFT agents to maximize the final task reward.
%(5) Direct Preference Optimization (DPO)~\cite{DPO23nips} is an RL method. 
We also compare \ours{} with Closed-Source LLMs including \textbf{GPT-3.5} (GPT-3.5-turbo-1106)~\cite{openai2024gpt3.5technicalreport} and \textbf{GPT-4} (GPT-4-0125-preview)~\cite{openai2024gpt4technicalreport}. % and Open-Source untuned and supervised fine-tuned LLMs containing Mistral$_{\text{7B}}$ and Llama3$_{\text{8B}}$. 

\subsection{Evaluation Metrics}
To align with previous methods~\cite{ETO24acl,SPIN24icml}, we report the average results of the test set. For WebShop and Science World, we employ the final reward automatically assessed by the environment as the evaluation metric while for ALFWorld, we utilize the success rate for judgement. In terms of Mind2Web, we report macro element accuracy. Additionally, for the three multi-hop question-answering tasks, we leverage Exact Match (EM) for evaluation. 

%In terms of Mind2Web, we report macro element accuracy as the primary evaluation metric. Additionally, for the three multi-hop question answering tasks, including HotpotQA, 2WikiMultihopQA, and MuSiQue, we leverage Exact Match (EM) for evaluation. 

\subsection{Implementation Details}
Consistent with existing works~\cite{ETO24acl,agentbench23arxiv}, we employ ReAct-form~\cite{ReAct23iclr} to generate the interaction trajectory, which additionally generates Chain-of-Thought (CoT) rationales~\cite{COT22nips} before each action. For each task, a one-shot in-context example is employed in the instruction prompt. The details of prompts are described in Appendix~\ref{sec:prompts}. For the three multi-hop question answering tasks, due to the lack of intermediate reasoning steps in the datasets, we employ GPT-4~\cite{openai2024gpt4technicalreport} as the expert to generate trajectories and select trajectories with the exact match score equalling one as the expert trajectories. We leverage greedy generation for our method and all baseline approaches. 
In the SFT stage, we set the learning rate as 1e-5 and the batch size as 64. we choose the cosine scheduler with a 0.03 warm up. We train the model for four epochs on all datasets. For the reflection stage, the learning rate is 5e-7 and the batch size is 16. The training epoch is set as one. We leverage the AdamW optimizer in both stages. All experiments are carried out on 8 NVIDIA A100 80G GPUs. 

\section{Results and Analysis}
%The best result is in bold and the second-best result is underlined.
\subsection{Overall Results}
\begin{table*}
  \caption{Performance comparison of all methods. Both variants can outperform all baselines based on open-sourced models.} 
  \small
  \centering
  \label{tab:performance}
  \begin{tabular}{llccccccc}
    \toprule
    \multirow{2}{*}[-0.5ex]{Backbone} & \multirow{2}{*}[-0.5ex]{Methods}& \multicolumn{2}{c}{\textbf{Web Tasks}} & \multicolumn{2}{c}{\textbf{Agent Tasks}} & \multicolumn{3}{c}{\textbf{Question-Answering Tasks}} \\%& \multirow{2}{*}[-0.5ex]{\makecell{Avg \\ (\%)}}\\
    \cmidrule(lr){3-4}\cmidrule(lr){5-6}\cmidrule(lr){7-9}
    %& \textbf{Method} & \textbf{10} & \textbf{15} & \textbf{20} & \textbf{Avg.} & \textbf{10} & \textbf{15} & \textbf{20} & \textbf{Avg.} \\
    & &WebShop & Mind2Web & Science World & ALFWorld & HotpotQA & 2WikiMultihopQA & MuSiQue \\
    \midrule
    %\multicolumn{8}{c}{\textit{Closed-Source LLMs}}\\
    %\multirow{3}{*}{
    GPT-3.5 & Base & 40.2 & 2.0 & 19.9 & 2.2 & 13.0 & 17.6 & 4.6\\
    %& Self-Refine \\
    %& Reflexion \\
    GPT-4 & Base & 58.0 & 26.7 & 53.6 & 36.6 & \textbf{39.4} & \textbf{64.8} & 28.2\\
    %& Self-Refine \\
    %& Reflexion \\
    %with React
    %SELF-RAG \\
    %Self-Refine & -\\
    \midrule
    %\colorbox{Thistle}{Vicuna$_{\text{7B}}$} & Base & 0.9 & & 3.7 & 0.0 & 2.8 & 4.6 & 0.4 \\
    \multirow{10}{*}{Mistral$_{\text{7B}}$} & Base & 2.7 & 17.8 & 4.2 & 0.0 & 4.2 & 10.6 & 1.4\\
    & SFT & 60.1 & 48.7 & 52.0 & 68.5 & 24.8 & 40.4 & 22.9\\
    %& BON & - & &59.7 & 68.7 & 29.2 & 69.6 & 24.0\\
    & PPO & 60.8 & 49.5 & 53.3 & 69.1 & 25.4 & 41.5 & 23.2\\
    & DPO & 62.4 & 50.9 & 54.1 & 70.6 & 26.9 & 42.7 & 24.9\\
    & RFT & 61.5 & 49.8 & 53.2 & 69.8 & 26.0 & 42.2 & 23.5\\
    & SPIN & 63.6 & 51.7 & 55.0 & 71.4 & 27.6 & 43.1 & 25.0\\
    & NAT & 61.3 & 50.4 & 52.9 & 69.3 & 26.1 & 41.9 & 24.1\\
    & ETO & 64.1 & 52.4 & 56.5 & 72.8 & 28.2 & 43.8 & 25.4 \\
    % & WKM \\
    %& IPR \\
    &\ours{}-Implicit & 66.2 & 53.3 & 59.6 & 74.2 & \textbf{31.0} & \textbf{46.8} & \textbf{27.7}\\
    &\ours{}-Inverse & \textbf{66.5} & \textbf{53.6} & \textbf{59.7} & \textbf{74.9} & 30.8 & 46.6 & 27.5\\
    %Baichuan2$_{\text{7B}}$\\
    %Qwen$_{\text{7B}}$\\
    %\colorbox{Thistle}{Mistral$_{\text{7B}}$} & \\
    %Baichuan2$_{\text{7B}}$\\
    %Qwen$_{\text{7B}}$\\
    \midrule
    \multirow{10}{*}{Llama3$_{\text{8B}}$} & Base & 7.2 & 23.6 & 32.3 & 0.0 & 15.6 & 13.8 & 9.0\\
    %\colorbox{Thistle}{Llama2$_{\text{7B}}$} & Base & 1.0 & & 3.0 & 0.0 & 0.8 & 3.2 & 0.0 \\
    %& \colorbox{Thistle}{SFT} & 61.1 & & 62.0 & 68.7 & 29.2 &69.6 & 24.0\\
    %& BON & - & &59.7 & 68.7 & 29.2 & 69.6 & 24.0\\
    & SFT & 62.6 & 50.3 & 54.5 & 67.8 & 33.0 & 47.8 & 30.6 \\
    & PPO & 63.2 & 51.0 & 55.0 & 67.9 & 33.2 & 47.6 & 30.4\\
    & DPO & 64.0 & 52.6 & 56.9 & 70.3 & 35.1 & 48.5 & 31.6\\
    & RFT & 63.6 & 50.8 & 54.7 & 68.0 & 33.5 & 47.9 & 30.8\\
    & SPIN& 65.4 & 53.9 & 60.3 & 71.9 & 34.8 & 48.9 & 31.9\\
    & NAT & 63.2 & 50.9 & 55.6 & 68.3 & 33.4 & 48.0 & 31.0\\
    & ETO & 65.7 & 54.0 & 62.5 & 73.4 & 35.2 & 49.4 & 32.3\\
    % & WKM & 66.6 & & 62.3 & 73.4\\
    %& IPR \\
    &\ours{}-Implicit & 67.2 & 55.8 & 63.6 & 75.5 & \textbf{38.1} & 51.3 & \textbf{34.4} \\
    &\ours{}-Inverse & \textbf{67.6} & \textbf{55.9} & \textbf{64.1} & \textbf{76.1} & 37.8 & \textbf{52.0} & 34.1 \\
    %\colorbox{Lavender}{-}
    %\colorbox{YellowGreen}{-}
    %\multirow{5}{*}{LLama-2-7B-Chat} & -\\
    %\midrule
    %\multicolumn{7}{c}{\textit{13B Parameters Open-Source LLMs}}\\
    %Vicuna$_{\text{13B}}$ \\
    %Mistral$_{\text{13B}}$ \\
    %Llama2$_{\text{13B}}$ \\
    %+SFT \\
    %+RFT \\
    %+Best-of-N \\
    %+PPO & \\
    %+DPO \\
    %+\ours{} \\
  \bottomrule
\end{tabular}
\end{table*}

\begin{table*}
  \caption{Ablation studies with different reward types based on Llama3$_{\text{8B}}$ and inverse reinforcement learning.}
  \label{tab:ab}
  \small
  \begin{tabular}{lccccccccc}
    \toprule
    \multirow{2}{*}[-0.5ex]{Method} & \multicolumn{2}{c}{\textbf{Reward Type}} & \multicolumn{2}{c}{\textbf{Web Tasks}} & \multicolumn{2}{c}{\textbf{Agent Tasks}} & \multicolumn{3}{c}{\textbf{Question-Answering Tasks}} \\ 
    %& \multirow{2}{*}{Average}\\
    \cmidrule(r){2-3}\cmidrule(lr){4-5}\cmidrule(lr){6-7}\cmidrule(l){8-10}
    % \cmidrule(lr){2-10}
    %\cmidrule(lr){2-3}\cmidrule(lr){4-5}\cmidrule(lr){6-7}\cmidrule(lr){8-10}
    & Step & Final &WebShop & Mind2Web & Science World & ALFWorld & HotpotQA & 2WikiMultihopQA & MuSiQue \\
    \midrule
    \ours{}-inverse & $\times$ & $\checkmark$ & 65.7 & 54.0 & 62.5 & 73.4 & 35.2 & 49.4 & 32.3\\
    \ours{}-inverse & $\checkmark$ & $\times$ & 67.6 & 55.9 & 64.1 & 76.1 & 37.8 & 52.0 & 34.1\\
    \ours{}-inverse & $\checkmark$ & $\checkmark$ & \textbf{68.0} & \textbf{56.4} & \textbf{64.8} & \textbf{76.2} & \textbf{38.9} & \textbf{52.0} & \textbf{34.8}\\
  \bottomrule
\end{tabular}
\end{table*}

The overall performance of our proposed methods \ours{} and all baselines are shown in Table~\ref{tab:performance}. We can observe that:

(1) Both variants of \ours{} consistently outperform all baseline methods across three distinct task categories by a significant margin. In comparison with ETO and SPIN, which introduce the entire trajectory for training, \ours{} achieves a significant edge with improvements of the results over all tasks. This performance demonstrates the effectiveness of utilizing the step-wise reward signals to emulate the expert policy. Even without human-annotated step-wise preference data, \ours{} still can gradually align with the expert policy distribution, leading to substantial enhancements in the response quality.

%\ours{}'s outstanding capability to address complex tasks through a step-wise learning approach emulating the expert policy. As a result, \ours{} is able to align with expert behaviors, even without human-annotated preference data, leading to substantial enhancements in the response quality.

%\ours{} construct step-wise data to compare the differences between the agent and the expert at each step to align with the expert. 
%demonstrating consistent and robust capabilities in tackling complex tasks through step-wise learning expert behaviors without dependency on human-annotated data. 
%The improvement in the results demonstrates that such an observational learning strategy can effectively alleviate the problem of insufficient diversity of user history. 

%compare DPO with GAIL
%compare different tasks
(2) Inverse reinforcement learning strategy \textbf{\ours{}-Inverse} with explicit rewards demonstrates a slight performance improvement compared to implicit-reward reinforcement learning methods \textbf{\ours{}-Implicit} on most datasets. This indicates that explicit rewards can provide the model with much clearer optimization objectives, thereby facilitating more effective adjustments in behavior. Consequently, the clarity of the optimization targets enables the novice agent to more effectively approach expert-level performance.

%, thereby enabling the model to approach expert-level performance. The clarity provided by explicit rewards enhances the model's ability to focus on essential task elements, promoting targeted improvements in its capabilities.

(3) Interestingly, \ours{} has achieved more significant improvements on the three multi-hop question-answering tasks. Concretely, \ours{} can surpass the state-of-the-art model ETO by an absolute value improvement of 2.9\% on the HotpotQA dataset. Since the reasoning steps in multi-hop question-answering tasks demonstrate complex semantic relationships (\ie, parallel or hierarchical), it is challenging for the agent to effectively imitate such complex expert policy based solely on final reward signals (\eg, task success or failure or exact match with the correct answer). Introducing step-wise rewards can facilitate the agent's deeper understanding and internalization of the expert policy's underlying logic.

%This finding confirms the necessity of introducing step-wise reward signals into the optimization process. facilitate identifying potential shortcomings in user behaviors and addressing them immediately. This verifies the limitations of merely depending on the final reward signal (\eg, the success or failure of a task or the exact match with the correct answer) to optimize and update agent strategies. In contrast, \ours{} constructs and utilizes step-wise reward, which can facilitate identifying potential shortcomings in user behaviors and addressing them immediately. 

%the absolute value of recall@10 on Cell Phones \& Accessories is increased by 0.75\% over the best GNN-based recommendation method LGC. 
%LGC utilizes accuracy as its only optimization goal, which leads to a decrease in several diversity evaluation metrics. In contrast, \ours{} simultaneously promotes accuracy and diversity and can achieve a good trade-off. 

In the following sections, we conduct several additional experiments to investigate \ours{} in depth. 
%These analyses are mainly based on Webshop and we can obtain similar conclusions on other datasets. 

\subsection{Ablation Studies on Reward Type}

In this section, we conduct ablation studies to analyze the influence of different reward types on our \ours{} model. We investigate our model \ours{} with three variants~\cite{rewardtype22arxiv,rewardtype23arxiv}: (1) Step-wise reward, which constructs step-wise reward by observing and imitating the expert behaviors and optimizes agent strategy with step-wise rewards, as introduced in Section~\ref{sec:method}. (2) Final reward, which utilizes the final environmental feedback as the reward for optimization. (3) We also explore the combination of the two reward types to evaluate their impact on the performance of \ours{}. Experiments are conducted based on Llama3$_{\text{8B}}$ and inverse RL and we can obtain similar conclusions with other settings. 

%This indicates that step-wise reward feedback are beneficial for agent optimization. 

From the results in Table~\ref{tab:ab}, we can observe that optimizing the reinforcement learning process solely with the final environmental feedback as rewards results in performance degradation on all tasks. Concretely, eliminating step-wise reward causes the obvious drop on all tasks (\eg, WebShop: 68.0$\xrightarrow{}$67.2 and Science World: 64.8$\xrightarrow{}$63.6). This indicates that the step-wise reward can facilitate the agent's capability to align with the expert. Meanwhile, the final environmental feedback also contributes to the final results which verifies that a combination of the step-wise and the final reward supervision is beneficial. Step-wise reward supervision provides immediate feedback, boosting optimization efficiency, while the final supervision offers clear direction for the overall learning objectives. Although the combination of the two reward types can lead to better results, obtaining the final reward necessitates interaction with the environment which cannot be parallelized. Consequently, in this paper, we exclusively focus on adopting the step-wise reward to strike a balance between efficiency and effectiveness.

%compared with utilizing the final environment feedback as reward signal, \ours{} with fine-grained step-wise reward supervision signals 

%Meanwhile, the attention process contributes a lot to the final results which verifies the importance of selecting suitable imitated users. In our experiments, we choose users with similar preferences but unique category interactions to learn which can improve diversity without sacrificing accuracy. %relieve the accuracy-diversity dilemma. % and provide better recommendation results. 
%Additionally, we have observed that without purification leads to a severe drop indicating denoising and distilling meaningful periodic characteristics can enhance remember. 

\subsection{Performance with Different Model Size}

\iffalse
\begin{table}
  \caption{Different model size. }%We remove (1) diversified curriculum user selection, (2) learnable denoising blocks, and (3) curriculum knowledge transfer block. }
  \label{tab:size}
  %\small
  \begin{tabular}{llll}
    \toprule
    Backbone & WebShop & Science World & HotpotQA \\
    \midrule
    \multirow{2}{*}{Mistral$_{\text{7B}}$} \\
    & \\
    %&\ours{}-DPO & 66.2 & 53.3 & 59.6 & 74.2 & 31.0 & 46.8 & 27.7\\
    %&\ours{}-GAIL & 66.5 & 53.6 & 59.7 & 74.9 & 30.8 & 46.6 & 27.5\\
    \midrule
    \multirow{2}{*}{Mistral$_{\text{13B}}$} \\
    \\
    %3个点左右
    %&\ours{}-DPO & 69.1 & 56.8 & 62.6 & 77.7 & 34.6 & 50.5 & 31.6\\
    %&\ours{}-GAIL & 69.9 & 57.4 & 63.2 & 78.1 & 34.1 & 50.6 & 31.7 \\
    \bottomrule
\end{tabular}
\end{table}
\fi

To further illustrate the robustness of \ours{}, we conduct experiments with different backbone model parameter sizes. We utilize Mistral$_\text{7B}$ and Mistral$_\text{13B}$\footnote{Mistral-Nemo-Instruct-2407} for this analysis. The results are depicted in Figure~\ref{fig:model_size}. ``-Implicit'' indicates \ours{} with implicit-reward reinforcement learning strategy while ``-Inverse'' represents inverse reinforcement learning method. We abbreviate Science World and 2WikiMultihopQA as Sci-World and 2WikiQA for limited space.

First, we can observe that \ours{} demonstrates consistent and robust efficacy across models with different parameter scales. This performance stability highlights our model's adaptability to different configurations, ensuring that its reliability in achieving effective results regardless of the parameter scale employed. Second, compared with Mistral$_\text{7B}$, Mistral$_\text{13B}$ achieves superior performances. This indicates the importance of the backbone model's capability, as it significantly influences the effectiveness of post-imitated learning. The enhanced capacity of Mistral$_\text{13B}$ allows for more effective learning and adaptation, contributing to improved performances.

\begin{figure}
    \centering
    \includegraphics[width=\linewidth]{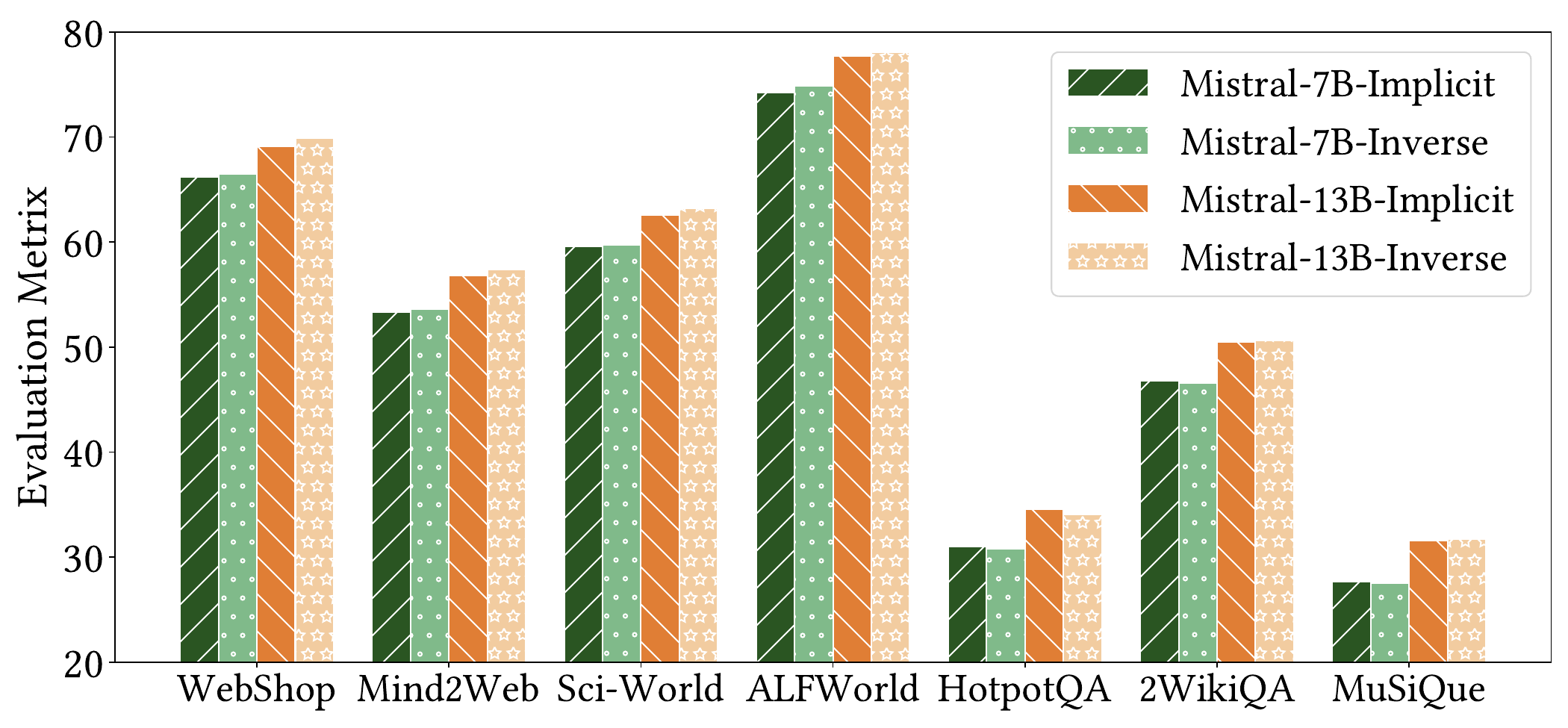}
    \caption{Performance with different backbone model parameters on all datasets. }
    \label{fig:model_size}
\end{figure}

\subsection{Exploration of Parameters Settings}
In \ours{}, two important hyper-parameters will impact the experimental performance -- the number of training iterations in Algorithm~\ref{Alg-IRL} and the practice number of the agent during the inspection stage of each iteration.
%In \ours{}, the experimental performance are significantly impacted by the maximum number of training iterations in Algorithm~\ref{Alg-IRL} and the practice number of the agent during the inspection stage. 
In this section, we conduct experiments to investigate their influences. We randomly selected two representative datasets WebShop and HotpotQA for this experiment and we can draw similar conclusions on other datasets. 

\textbf{Training Iteration.} 
To identify the optimal iteration number, we increase the training iteration number from one to nine, while closely monitoring the performance changes associated with the two reflection mechanisms. 
As depicted in Figure~\ref{fig:iteration}(a), the performance of \ours{} improves progressively as the training iteration number increases for both implicit-reward and inverse reinforcement learning strategies. However, the peak performance of the two methods differs. Specifically, on the WebShop dataset, the implicit-reward strategy reaches the peak after three iterations whereas the inverse reinforcement learning method achieves its best performance at the seven iteration. This indicates that more iterations are required for the model to correctly learn the explicit reward function, which leads to slower convergence. Besides, the performance starts to degrade when the iterations exceed the peak. This phenomenon can be attributed to the fact that as the agent's capabilities improve, our self-play method for generating step-wise fine-tuning data may struggle to provide contrasting positive and negative samples. The absence of clear distinctions between successful and unsuccessful behaviors disrupts the learning process. %Consequently, this limitation may slow down the convergence and overall performance of the model as it attempts to refine its strategies.

%for both method, once the iteration count reaches 2 and 3, the performance peaks, indicating that step-wise learning expert behaviors can indeed enhance inference accuracy. Nevertheless, excessively increasing the number of iteration rounds leads to a slight decline in results. This could be attributed to the model’s diminishing ability to extract more useful information or suggestions through the metacognitive mechanism. An intriguing observation is a minor accuracy peak at an iteration count of 2. This phenomenon primarily arises from the characteristics of the 2WikiMultihopQA dataset, where the majority of questions require references from two sources. Two rounds of metacognitive reflection prove sufficient to gather the necessary knowledge for these questions. Beyond that, additional rounds of reflection tend to introduce noise, resulting in fluctuating results.

\begin{figure}
    \centering
    \includegraphics[width=\linewidth]{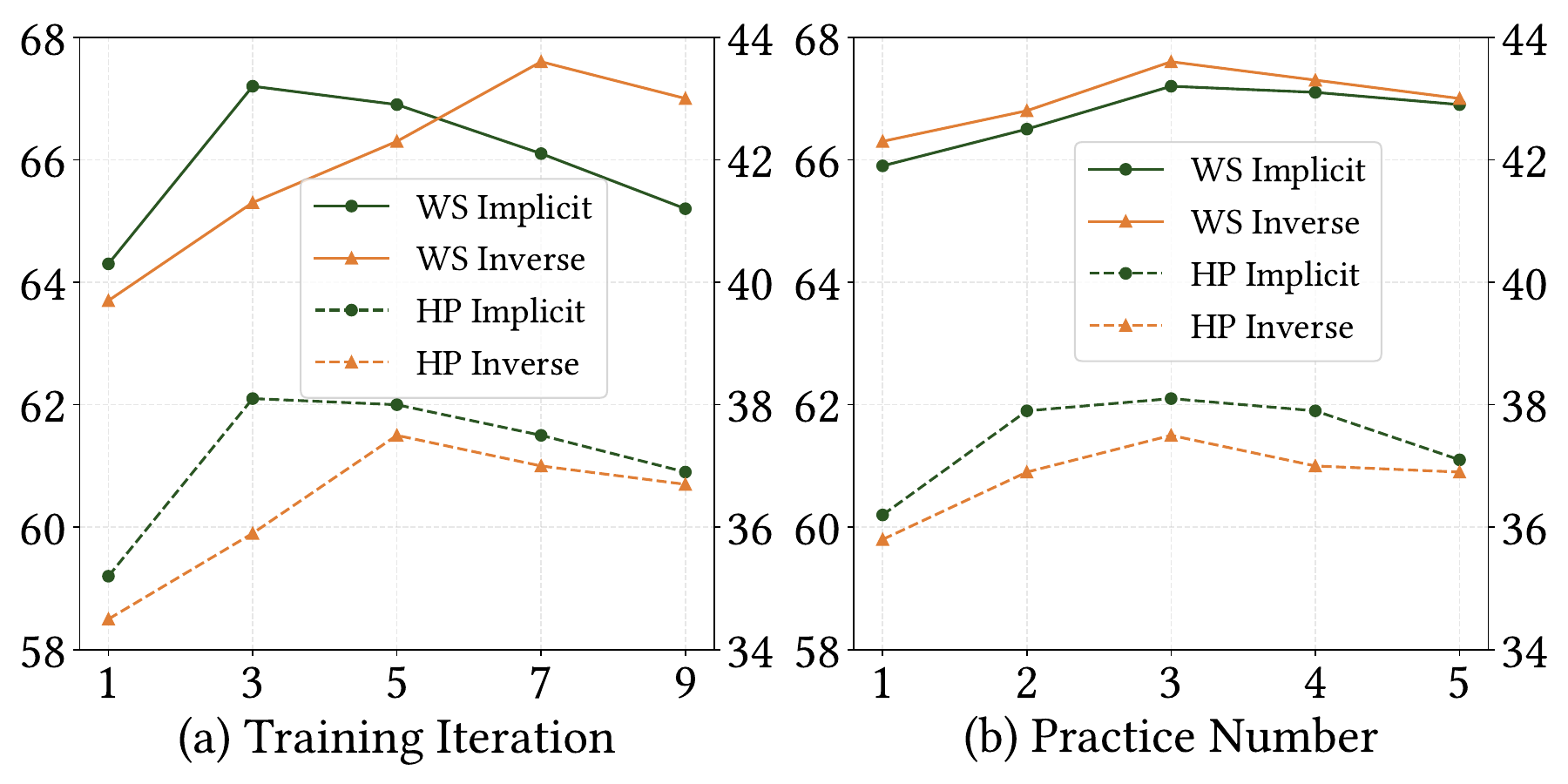}
    \caption{Performance with different training iterations and practice numbers. ``WS'' is WebShop while ``HP'' is HotpotQA. }
    \label{fig:iteration}
\end{figure}
% 64.3,67.2,66.9,66.1,65.2
% 63.7,65.3,66.3,67.6,67.0
% 35.2,38.1,38.0,37.5,36.9
% 34.5,35.9,37.5,37.0,36.7

% 65.9,66.5,67.2,67.1,66.9
% 66.3,66.8,67.6,67.3,67.0
% 36.2,37.4,38.1,37.9,37.1
% 35.8,36.9,37.5,37.0,36.9

\textbf{Practice Number.}
In this part, we conduct experiments to investigate whether introducing diverse agent trajectories is beneficial for performance improvement. To achieve this, we force the novice agent to practice multiple times for each learning objective during the inspection phase. Figure~\ref{fig:iteration}(b) shows the results of two reflection variants. We can observe that the results of both variants are gradually increasing as the practice number grows from one to three. This implies that introducing more diverse training samples can accelerate the novice's acquisition of the expert policy. However, the performance does not increase when the practice number exceeds three. A potential explanation is that, at the same cognitive level, the diversity of the samples remains limited despite multiple attempts. Consequently, incorporating more samples may lead to information redundancy, which can hinder learning efficiency and also increase computational costs.

%imitate the expert's actions with different practice number from one to five and track the performance changes. 
%performances can be improved by introducing diverse agent trajectories by practicing multiple times during the inspection phase. To achieve this, we force the novice agent to imitate the expert's actions with different practice number from one to five and track the performance changes. 
%sampling model behaviors multiple times during the inspection phase to increase the diversity of negative samples. 
%Figure~\ref{fig:iteration}(b) shows the results of two reflection variants. We can observe that the performance of both strategies are gradually increasing as the practice number grows from one to three. This implies that, in the initial stage of training, introducing more diverse training samples significantly aids the novice in mastering the expert policy more rapidly. However, the performance does not increase when the practice number exceeds three. The potential reason may be that at the same cognitive level, the diversity of the samples remains limited despite multiple samplings. At this point, incorporating more samples may lead to information redundancy and also increase computational costs.

\section{Conclusion and Future Work}
Reinforcement learning has become an effective approach for aligning agent behaviors with human preferences. However, existing reinforcement learning methods primarily adopt the final environmental feedback to optimize the agent strategy. In this paper, inspired by Benner's novice-to-expert theory, we proposed \ours{}, a step-wise reinforcement learning framework without step-wise human annotation. In the inspection stage, the novice agent first observes the behaviors of the expert and then rehearses the demonstrated actions. During the reflection stage, the agent compares its actions with those of the expert and adjusts its policy to better align with the expert's policy distribution. Experimental results across three types of tasks consistently demonstrate the superiority of \ours{} over existing baselines. Besides, we conduct additional experiments to further illustrate the effectiveness and efficiency of \ours{}. 
In the future, we aim to enhance LLM agents by integrating more advanced cognitive capabilities to better satisfy user demands and respond to dynamic environments. 
%This will enable agents to better understand user demands and respond to dynamic environments. %, thereby achieving a deeper understanding of the human behavior and society.

%In the future, we aspire to enhance LLM agents by integrating more advanced cognitive capabilities, such as adaptive learning, contextual awareness, and collaborative problem-solving. This will enable agents to better understand user demands and respond to dynamic environments, thereby achieving a deeper understanding of human behavior and social nuances.

%%
%% The next two lines define the bibliography style to be used, and
%% the bibliography file.
\bibliographystyle{ACM-Reference-Format}
\bibliography{sample-base}

%%
%% If your work has an appendix, this is the place to put it.
\appendix
\section{Prompts}
\label{sec:prompts}
\subsection{WebShop}
\begin{cvbox}[Task Instruction for WebShop]
You are web shopping.
I will give you instructions about what to do.
You have to follow the instructions.
Every round I will give you an observation and a list of available actions, you have to respond an action based on the state and instruction.
You can use search action if search is available.
You can click one of the buttons in clickables.

An action should be of the following structure:

search[keywords]

click[value]

If the action is not valid, perform nothing.
Keywords in search are up to you, but the value in click must be a value in the list of available actions.
Remember that your keywords in search should be carefully designed.

Your response should use the following format:

Thought: I think ...

Action: search[something]
\end{cvbox}

\subsection{Mind2Web}
\begin{cvbox}[Task Instruction for Mind2Web]
You are a helpful assistant that is great at website design, navigation, and executing tasks for the user.

\textbf{User}: 
"<html> <div> <div> <a tock home page /> <button id=0 book a reservation. toggle open> <span> Book a reservation </span> </button> <button book a reservation. toggle open> </button> </div> <div> <select id=1 type> <option reservations true> Dine in </option> <option pickup> Pickup </option> <option delivery> Delivery </option> <option events> Events </option> <option wineries> Wineries </option> <option all> Everything </option> </select> <div id=2> <p> Celebrating and supporting leading women shaking up the industry. </p> <span> Explore now </span> </div> </div> </div> </html>" Based on the HTML webpage above, try to complete the following task: 

Task: Check for pickup restaurant available in Boston, NY on March 18, 5pm with just one guest

Previous actions: None

What should be the next action? Please select from the following choices (If the correct action is not in the page above, please select A. 'None of the above'):

A. None of the above

B. <button id=0 book a reservation. toggle open> <span> Book a

C. <select id=1 type> <option reservations true> Dine in </option> <option

D. <div id=2> <p> Celebrating and supporting leading women shaking up

\textbf{Assistant:}
Answer: C. Action: SELECT, Value: Pickup

\textbf{User}: 
"<html> <div> <main main> <section tabpanel> <div> <ul tablist> <li tab heading level 3 search and> </li> <li id=0 tab heading level 3 search and> <span> Hotel </span> </li> <li tab heading level 3 search and> </li> <li tab heading level 3 search and> </li> </ul> <div tabpanel> <div id=1> <div> <span> Dates* </span> <button button clear dates /> </div> <div> <label> Travelers </label> <div> <p> 1 Adult </p> <button button> 1 Adult </button> <div dialog> <button button travel with a pet. this> <span> Travel with a pet </span> </button> <div> <button button clear all fields> Clear all </button> <button button> </button> </div> </div> </div> </div> </div> </div> </div> </section> </main> <footer contentinfo> <div> <h3> Stay Connected </h3> <ul id=2> <a mobile tools> </a> <a open united's tiktok feed in> </a> <a open united's facebook page in> </a> <a open united's twitter feed in> </a> <a open united's youtube page in> </a> <a open united's instagram feed in> </a> <a open united's linkedin profile in> </a> </ul> </div> </footer> </div> </html>" Based on the HTML webpage above, try to complete the following task:

Task: Compare the fare types to book a 1-adult ticket from Springfiels, IL to Austin, TX for April 29th 2023

Previous actions: [combobox]  Enter your departing city, airport name, or airpor... -> TYPE: SPRINGFIELD [button]  Springfield, IL, US (SPI) -> CLICK [combobox]  Enter your destination city, airport name, or airp... -> TYPE: AUSTIN [button]  Austin, TX, US (AUS) -> CLICK

What should be the next action? Please select from the following choices (If the correct action is not in the page above, please select A. 'None of the above'):

A. None of the above

B. <li id=0 tab heading level 3 search and> <span> Hotel

C. <div id=1> <div> <span> Dates* </span> <button button clear dates

D. <ul id=2> <a mobile tools> </a> <a open united's tiktok"

\textbf{Assistant}:
Answer: A.

\textbf{User}: 
"<html> <div> <nav main menu> <ul> <li> <div button> Car Sales </div> <div id=0> <div> <div> <div> Buy A Car </div> <div> Plan Your Purchase </div> </div> <div> <h4> Its Tax Refund Time. Treat Yourself to an Upgrade. </h4> <p> With a variety of options, invest your refund in what you really want - a quality, used vehicle from Enterprise. </p> <a> View Inventory </a> </div> </div> </div> </li> <div id=1> Enterprise Fleet Management </div> </ul> </nav> <div region> <button id=2 selected pick-up date 03/19/2023> <span> <span> 19 </span> <div> <span> Mar </span> <span> 2023 </span> </div> </span> </button> </div> </div> </html>" Based on the HTML webpage above, try to complete the following task:

Task: Find a mini van at Brooklyn City from April 5th to April 8th for a 22 year old renter.

Previous actions: [searchbox]  Pick-up \& Return Location (ZIP, City or Airport) (... -> TYPE: Brooklyn [option]  Brooklyn, NY, US Select -> CLICK

What should be the next action? Please select from the following choices (If the correct action is not in the page above, please select A. 'None of the above'):

A. None of the above

B. <div id=0> <div> <div> <div> Buy A Car </div> <div>

C. <div id=1> Enterprise Fleet Management </div>

D. <button id=2 selected pick-up date 03/19/2023> <span> <span> 19 </span>

\textbf{Assistant}:
Answer: D. Action: CLICK
\end{cvbox}

\subsection{Science World}

\begin{cvbox}[Task Instruction for Science World]
You are a helpful assistant to do some scientific experiment in an environment.
In the environment, there are several rooms: kitchen, foundry, workshop, bathroom, outside, living room, bedroom, greenhouse, art studio, hallway
You should explore the environment and find the items you need to complete the experiment.
You can teleport to any room in one step.
All containers in the environment have already been opened, you can directly get items from the containers.
The available actions are:

open OBJ: open a container\\
close OBJ: close a container\\
activate OBJ: activate a device\\
deactivate OBJ: deactivate a device\\
connect OBJ to OBJ: connect electrical components\\
disconnect OBJ: disconnect electrical components\\
use OBJ [on OBJ]: use a device/item\\
look around: describe the current room\\
examine OBJ: describe an object in detail\\
look at OBJ: describe a container's contents\\
read OBJ: read a note or book\\
move OBJ to OBJ: move an object to a container\\
pick up OBJ: move an object to the inventory\\
pour OBJ into OBJ: pour a liquid into a container\\
mix OBJ: chemically mix a container\\
teleport to LOC: teleport to a specific room\\
focus on OBJ: signal intent on a task object\\
wait: task no action for 10 steps\\
wait1: task no action for a step

Your response should use the following format:

Thought: I think ...

Action: open OBJ
\end{cvbox}

\subsection{ALFWorld}
\begin{cvbox}[Task Instruction for ALFWorld]
Interact with a household to solve a task. Imagine you are an intelligent agent in a household environment and your target is to perform actions to complete the task goal. At the beginning of your interactions, you will be given the detailed description of the current environment and your goal to accomplish. 
For each of your turn, you will be given the observation of the last turn. You should first think about the current condition and plan for your future actions, and then output your action in this turn. 

The available actions are:

1. go to {recep}\\
2. task {obj} from {recep}\\
3. put {obj} in/on {recep}\\
4. open {recep}\\
5. close {recep}\\
6. toggle {obj} {recep}\\
7. clean {obj} with {recep}\\
8. heat {obj} with {recep}\\
9. cool {obj} with {recep}\\
where {obj} and {recep} correspond to objects and receptacles.
After your each turn, the environment will give you immediate feedback based on which you plan your next few steps. if the envrionment output "Nothing happened", that means the previous action is invalid and you should try more options.

Your response should use the following format:

Thought: <your thoughts>

Action: <your next action>
\end{cvbox}

\subsection{HotpotQA, 2WikimultihopQA and Musique}
\begin{cvbox}[Task Instruction for Multihop-QA Datasets]
You are an expert in this field. Please answer the question as simply and concisely as possible.
Every round I will give you an observation, you have to respond with interleaving Thought and Action steps.
Thought can reason about the current situation, and Action can be two types: 

(1) Search[entity], which searches the exact entity on Wikipedia and returns the first paragraph if it exists. If not, it will return some similar entities to search.

(2) Finish[answer], which returns the answer and finishes the task.

Your response should use the following format:

Thought: I think ...

Action: ...
\end{cvbox}

\iffalse

\section{Proof}
\label{sec:proof}
\begin{theorem}[Proposition 1]
\label{sec:proof1}
Given the regularizer
\[
    \varpi(c)\triangleq\left\{\begin{array}{ll}
       \mathbb{E}_{\pi_e}[-c(s,a)-\log(1-e^{c(s,a)})]  &  c<0\\
        +\infty & c\geq 0
    \end{array} \right.
\]
, then, \[
        \varpi^*(\rho_\pi-\rho_{\pi_e})=\max \mathbb{E}_\pi[\log(D(s,a))]
        +\mathbb{E}_{\pi_e}[\log(1-D(s,a))].
        \]
\end{theorem}

\begin{proof}
Following \citet{GAIL16nips}, we have 
\begin{equation}
    \small
    \begin{aligned}
        \varpi^*(\rho_\pi-\rho_{\pi_e})&=-R(\rho_\pi,\rho_{\pi_e})\\
    &=\sum_{s,a}\max_{\gamma\in R}\rho_\pi(s,a)\log(\frac{1}{1+\exp(-\gamma)})\\
    &\quad\quad\quad\quad+\rho_\pi(s,a)\log(\frac{1}{1+\exp(\gamma)})\\
    &=\sum_{s,a}\max_{\gamma\in R}\rho_\pi(s,a)\log(\frac{1}{1+\exp(-\gamma)})\\
    &\quad\quad\quad\quad+\rho_\pi(s,a)\log(1-\frac{1}{1+\exp(-\gamma)})\\
    &=\sum_{s,a}\max_{\gamma\in R}\rho_\pi(s,a)\log(\sigma(\gamma))+\rho_{\pi_e}\log(1-\sigma(\gamma))
    \label{eq:cost_proof_1-3}
    \end{aligned}
\end{equation}
where $\sigma(x)=\frac{1}{1+\exp(-x)}\in(0,1)$ is the logistic function. 

We define $D(s,a)=\sigma(\gamma)$, then Equation~(\ref{eq:cost_proof_1-3}) can be writen as:
\begin{equation}
    \small
    \begin{aligned}
        \varpi^*(\rho_\pi-\rho_{\pi_e}) &=\sum_{s,a}\max\rho_\pi(s,a)\log(D(s,a))+\rho_{\pi_e}\log(1-D(s,a))\\
        &=\max \mathbb{E}_\pi[\log(D(s,a))]+\mathbb{E}_{\pi_e}[\log(1-D(s,a))]
    \end{aligned}
\end{equation}
\end{proof}
\fi

\end{document}